  \documentclass[10pt]{article}
  \usepackage{amsmath,amsthm,amsfonts,amssymb,latexsym,graphicx}
  \usepackage[T2A]{fontenc}
  \usepackage[utf8]{inputenc}
  \usepackage[pdfpagemode=UseNone,pdfstartview=FitH]{hyperref}

  \newcommand{\OCMI}{Vovk/etal:2009AOS}

  \newcommand{\OCMVII}{Vovk/Petej:arXiv1211}

  \newcommand{\OCMXVII}{Vovk/etal:2017COPA}

  \newcommand*{\GTPL}{GTP50}

\allowdisplaybreaks[4]

  \newtheorem{theorem}{Theorem}
  \newtheorem{lemma}[theorem]{Lemma}
  \theoremstyle{definition}
  \newtheorem{definition}[theorem]{Definition}
  \newtheorem{example}[theorem]{Example}
  \newtheorem{remark}[theorem]{Remark}

\DeclareMathOperator{\Prob}{\mathbb{P}}
\DeclareMathOperator{\Expect}{\mathbb{E}}

\newcommand{\R}{\mathbb{R}}
\newcommand{\PPP}{\mathcal{P}}
\newcommand{\FFF}{\mathcal{F}}
\newcommand{\GGG}{\mathcal{G}}

\newcommand{\dd}{\mathrm{d}}

\usepackage{accents}
\newlength{\dhatheight}

  \title{Universally consistent predictive distributions}
  \author{Vladimir Vovk}

\begin{document}
  \maketitle

\begin{abstract}
  This paper describes simple universally consistent procedures of probability forecasting
  that satisfy a natural property of small-sample validity,
  under the assumption that the observations are produced independently in the IID fashion.

    \bigskip

    \noindent
    The version of this paper at \href{http://alrw.net}{http://alrw.net}
    (Working Paper 18, first posted on 17 April 2017) is updated most often.
\end{abstract}

\section{Introduction}
\label{sec:introduction}

Predictive distributions are probability distributions for a future value of a response variable
satisfying a natural property of validity.
They were introduced independently by Schweder and Hjort \cite[Chapter~12]{Schweder/Hjort:2016}
and Shen et al.\ \cite{Shen/etal:2017},
who also gave several examples of predictive distributions in parametric statistics.
Earlier, related notions had been studied extensively by Tilmann Gneiting with co-authors
and their predecessors
(see, e.g., the review \cite{Gneiting/Katzfuss:2014}).
First nonparametric predictive distributions were constructed in \cite{\OCMXVII}
based on the method of conformal prediction
(see, e.g., \cite{Vovk/etal:2005book,\OCMI,Lei/etal:2013,Lei/Wasserman:2014}).
The nonparametric statistical model used in \cite{\OCMXVII} is the one that is standard in machine learning:
the observations are produced independently from the same probability measure;
we will refer to it as the \emph{IID model} in this paper (as we did in \cite{\OCMI}).
To make the notion of predictive distributions applicable in the nonparametric context,
\cite{\OCMXVII} slightly generalizes it allowing randomization;
unless the amount of training data is very small, randomization affects the predictive distribution very little,
but it simplifies definitions.

This paper follows \cite{\OCMXVII} in studying randomized predictive distributions under the IID model.
Namely, we construct randomized predictive distributions that,
in addition to the small-sample property of validity that is satisfied automatically,
satisfy an asymptotic property of universal consistency;
informally, the true conditional distribution of the response variable 
and the randomized predictive distribution for it computed from the corresponding predictor and training data of size $n$
approach each other as $n\to\infty$.
(The procedures studied in \cite{\OCMXVII} were based on the Least Squares procedure
and far from universally consistent;
cf.\ Example~\ref{ex:LSPM} below.)

Our approach is in the spirit of Gneiting et al.'s \cite{Gneiting/etal:2007} paradigm
(which they trace back to Murphy and Winkler \cite{Murphy/Winkler:1987})
of \emph{maximizing the sharpness of the predictive distributions subject to calibration}.
We, however, refer to calibration as validity, sharpness as efficiency,
and include a validity requirement in the definition of predictive distributions
(following Shen et al.\ \cite{Shen/etal:2017}).

We are mostly interested in results about the existence (and in explicit constructions) of randomized predictive distributions
that satisfy two appealing properties:
the small-sample property of validity and the asymptotic property of universal consistency.
However, if we do not insist on the former, randomization becomes superfluous (Theorem~\ref{thm:PFS}).

As in \cite{\OCMXVII}, our main technical tool will be conformal prediction.
Before \cite{\OCMXVII},
conformal prediction was typically applied for computing prediction sets.
Conformal predictors are guaranteed to satisfy a property of validity,
namely the right coverage probability,
and a remaining desideratum is their efficiency,
namely the smallness of their prediction sets.
Asymptotically efficient conformal predictors
were constructed by Lei et al.\ \cite{Lei/etal:2013} in the unsupervised setting
and Lei and Wasserman \cite{Lei/Wasserman:2014} in the supervised setting (namely, for regression).
This paper can be considered another step in this direction,
where the notion of efficiency is formalized as universal consistency.

For convenience, in this paper we will refer to procedures producing randomized predictive distributions as predictive systems;
in particular, conformal predictive systems are procedures producing conformal predictive distributions,
i.e., randomized predictive systems obtained by applying the method of conformal prediction.

The main result of this paper (Theorem~\ref{thm:CPS}) is that there exists a universally consistent conformal predictive system,
in the sense that it produces predictive distributions
that are consistent under any probability distribution for one observation.
The notion of consistency is used in an unusual situation here,
and our formalization is based on Belyaev's
  \cite{Belyaev:1995,Belyaev/Sjostedt:2000,Sjostedt:2005}\
notion of weakly approaching sequences of distributions.
The construction of a universally consistent conformal predictive system
adapts standard arguments for universal consistency in classification and regression
\cite{Stone:1977,Devroye/etal:1996,Gyorfi/etal:2002}.

We start in Section~\ref{sec:RPD} from defining randomized predictive systems,
which are required to satisfy the small-sample property of validity under the IID model.
The next section, Section~\ref{sec:CPD}, defines conformal predictive systems,
which are a subclass of randomized predictive systems.
The main result of the paper, Theorem~\ref{thm:CPS} stated in Section~\ref{sec:UC-CPD},
requires a slight generalization of conformal predictive systems
(for which we retain the same name).
Section~\ref{sec:MPD} introduces another subclass of randomized predictive systems,
which is wider than the subclass of conformal predictive systems of Section~\ref{sec:CPD};
the elements of this wider subclass are called Mondrian predictive systems.
A simple version of Theorem~\ref{thm:CPS} given in Section~\ref{sec:UC-MPD} (Theorem~\ref{thm:MPS})
states the existence of Mondrian predictive systems that are universally consistent.
An example of a universally consistent Mondrian predictive system is given in Section~\ref{sec:HMPD},
and Section~\ref{sec:proof-Mondrian} is devoted to a short proof that this predictive system is indeed universally consistent.
Section~\ref{sec:proof-PFS} gives an even shorter proof of the existence
of a universally consistent probability forecasting system (Theorem~\ref{thm:PFS}),
which is deterministic and not required to satisfy any small-sample properties of validity.
Theorem~\ref{thm:CPS} stated in Section~\ref{sec:UC-CPD} asserts the existence of universally consistent conformal predictive systems.
An example of such a conformal predictive system is given in Section~\ref{sec:HCPD},
and it is shown in Section~\ref{sec:proof-conformal} to be universally consistent.
One advantage of Theorem~\ref{thm:CPS} over the result of Section~\ref{sec:UC-MPD} (Theorem~\ref{thm:MPS}) is that,
as compared with Mondrian predictive systems,
conformal predictive systems enjoy a stronger small-sample property of validity
(see Remarks~\ref{rem:R2} and~\ref{rem:advantage}).
In conclusion, Section~\ref{sec:conclusion} lists some natural directions of further research.

\begin{remark}
  There is a widely studied sister notion to predictive distributions
  with a similar small-sample guarantee of validity,
  namely confidence distributions: see, e.g., \cite{Xie/Singh:2013}.
  Both confidence and predictive distributions go back to Fisher's fiducial inference.
  Whereas, under the nonparametric IID model of this paper,
  there are no confidence distributions,
  \cite{\OCMXVII} and this paper argue that there is a meaningful theory of predictive distributions
  even under the IID model.
\end{remark}

\section{Randomized predictive distributions}
\label{sec:RPD}

In this section we give some basic definitions partly following \cite{Shen/etal:2017} and \cite{\OCMXVII}.
Let $\mathbf{X}$ be a measurable space, which we will call the \emph{predictor space}.
The \emph{observation space} is defined to be $\mathbf{Z}:=\mathbf{X}\times\R$;
its element $z=(x,y)$, where $x\in\mathbf{X}$ and $y\in\R$, is interpreted as an \emph{observation}
consisting of a \emph{predictor} $x\in\mathbf{X}$
and a \emph{response variable} (or simply \emph{response}) $y\in\R$.
Our task is, given \emph{training data} consisting of observations $z_i=(z_i,y_i)$, $i=1,\ldots,n$,
and a new (test) predictor $x_{n+1}\in\mathbf{X}$,
to predict the corresponding response $y_{n+1}$;
the pair $(x_{n+1},y_{n+1})$ will be referred to as the test observation.
We will be interested in procedures whose output is independent of the ordering of the training data $(z_1,\ldots,z_n)$;
therefore, the training data can also be interpreted as a multiset rather than a sequence.

Let $U$ be the uniform probability measure on the interval $[0,1]$.
\begin{definition}\label{def:RPS}
  A measurable function
  $Q:\cup_{n=1}^{\infty}(\mathbf{Z}^{n+1}\times[0,1])\to[0,1]$
  is called a \emph{randomized predictive system} if it satisfies the following requirements:
  \begin{itemize}
  \item[R1]
    \begin{itemize}
    \item[i]
      For each $n$, each training data sequence $(z_1,\ldots,z_n)\in\mathbf{Z}^n$,
      and each test predictor $x_{n+1}\in\mathbf{X}$,
      the function $Q(z_1,\ldots,z_n,(x_{n+1},y),\allowbreak\tau)$ is monotonically increasing in both $y$ and $\tau$
      (i.e., monotonically increasing in $y$ for each $\tau$ and monotonically increasing in $\tau$ for each~$y$).
    \item[ii]
      For each $n$, each training data sequence $(z_1,\ldots,z_n)\in\mathbf{Z}^n$,
      and each test predictor $x_{n+1}\in\mathbf{X}$,
      \begin{align}
        \lim_{y\to-\infty}
        Q(z_1,\ldots,z_n,(x_{n+1},y),0)
        &=
        0,
        \label{eq:R1-first}\\
        \lim_{y\to\infty}
        Q(z_1,\ldots,z_n,(x_{n+1},y),1)
        &=
        1.
        \notag
      \end{align}
    \end{itemize}
  \item[R2]
    For each $n$,
    the distribution of $Q$,
    as function of random training observations $z_1\sim P$,\ldots, $z_n\sim P$,
    a random test observation $z_{n+1}\sim P$,
    and a random number $\tau\sim U$,
    all assumed independent,
    is uniform:
    \begin{equation}\label{eq:R2}
      \forall \alpha\in[0,1]:
      \Prob
      \left(
        Q(z_1,\ldots,z_n,z_{n+1},\tau)
        \le
        \alpha
      \right)
      =
      \alpha.
    \end{equation}
  \end{itemize}
  The function $Q(z_1,\ldots,z_n,(x_{n+1},\cdot),\tau)$ is the \emph{predictive distribution (function)}
  output by $Q$ for given training data $z_1,\ldots,z_n$, test predictor $x_{n+1}$, and $\tau\in[0,1]$.
\end{definition}

Requirement R1 says, essentially, that, as a function of $y$,
$Q$ is a distribution function, apart from a slack caused by the dependence on the random number $\tau$.
The size of the slack is
\begin{equation}\label{eq:difference}
  Q(z_1,\ldots,z_n,(x_{n+1},y),1)
  -
  Q(z_1,\ldots,z_n,(x_{n+1},y),0)
\end{equation}
(remember that $Q$ is monotonically increasing in $\tau\in[0,1]$, according to requirement R1(i)).
In typical applications the slack will be small unless there is very little training data;
see Remark~\ref{rem:difference} for details.

Requirement~R2 says, informally, that the predictive distributions
agree with the data-generating mechanism.
It has a long history in the theory and practice of forecasting.
The review \cite{Gneiting/Katzfuss:2014} refers to it as probabilistic calibration
and describes it as critical in forecasting;
\cite[Section~2.2.3]{Gneiting/Katzfuss:2014} reviews the relevant literature.

\begin{remark}
  Requirements R1 and R2 are the analogues
  (introduced in \cite[Chapter~12]{Schweder/Hjort:2016} and \cite{Shen/etal:2017})
  of similar requirements
  in the theory of confidence distributions:
  see, e.g., \cite[Definition~1]{Xie/Singh:2013} or \cite[Chapter~3]{Schweder/Hjort:2016}.
\end{remark}

\begin{definition}\label{def:consistency}
  Let us say that a randomized predictive system $Q$ is \emph{consistent}
  for a probability measure $P$ on $\mathbf{Z}$ if,
  for any bounded continuous function $f:\R\to\R$,
  \begin{equation}\label{eq:goal-1}
    \int f \dd Q_n - \Expect_P(f\mid x_{n+1})
    \to
    0
    \qquad
    (n\to\infty)
  \end{equation}
  in probability,
  where:
  \begin{itemize}
  \item
    $Q_n$ is the predictive distribution
    $
      Q_n: y\mapsto Q(z_1,\ldots,z_n,(x_{n+1},y),\tau)
    $
    output by $Q$ as its forecast for the response $y_{n+1}$ corresponding to the test predictor $x_{n+1}$
    based on the training data $(z_1,\ldots,z_n)$, where $z_i=(x_i,y_i)$;
  \item
    $\Expect_P(f\mid x_{n+1})$ is the conditional expectation of $f(y)$ given $x=x_{n+1}$
    under $(x,y)\sim P$;
  \item
    $z_i=(x_i,y_i)\sim P$, $i=1,\ldots,n+1$, and $\tau\sim U$, are assumed all independent.
  \end{itemize}
\end{definition}

It is clear that the notion of consistency given in Definition~\ref{def:consistency}
does not depend on the choice of the version of the conditional expectation $\Expect_P(f\mid\cdot)$ in~\eqref{eq:goal-1}.
The integral in \eqref{eq:goal-1} is not quite standard since we did not require $Q_n$ to be exactly a distribution function,
so we understand $\int f \dd Q_n$ as $\int f \dd \bar Q_n$
with the measure $\bar Q_n$ on $\R$ defined by $\bar Q_n((u,v]):=Q_n(v+)-Q_n(u+)$
for any interval $(u,v]$ of this form in $\R$.
\begin{definition}\label{def:universal-consistency}
  A randomized predictive system $Q$ is \emph{universally consistent}
  if it is consistent for any probability measure $P$ on $\mathbf{Z}$.
\end{definition}
As already mentioned in Section~\ref{sec:introduction},
Definition~\ref{def:universal-consistency} is based on Belyaev's (see, e.g., \cite{Belyaev/Sjostedt:2000}).
Our goal is construction of universally consistent randomized predictive systems.

\section{Conformal predictive distributions}
\label{sec:CPD}

A way of producing randomized predictive distributions under the IID model has been proposed in \cite{\OCMXVII}.
This section reviews a basic version,
and Section~\ref{sec:UC-CPD} introduces a simple extension.

\begin{definition}\label{def:conformity-measure}
  A \emph{conformity measure} is a measurable function
  $A:\cup_{n=1}^{\infty}\mathbf{Z}^{n+1}\to\R$
  that is invariant with respect to permutations of the training observations:
  for any $n$, any sequence $(z_1,\ldots,z_{n})\in\mathbf{Z}^n$, any $z_{n+1}\in\mathbf{Z}$,
  and any permutation $\pi$ of $\{1,\ldots,n\}$,
  \begin{equation*}
    A(z_1,\ldots,z_n,z_{n+1})
    =
    A\left(z_{\pi(1)},\ldots,z_{\pi(n)},z_{n+1}\right).
  \end{equation*}
\end{definition}

The standard interpretation of a conformity measure $A$
is that the value $A(z_1,\ldots,z_n,z_{n+1})$ measures how well the new observation $z_{n+1}$
conforms to the comparison data $(z_1,\ldots,z_n)$.
In the context of this paper, and conformal predictive distributions in general,
$A(z_1,\ldots,z_n,z_{n+1})$, where $z_{n+1}=(x_{n+1},y_{n+1})$,
measures how large the response variable $y_{n+1}$ is,
in view of the corresponding predictor $x_{n+1}$ and comparison data $z_1,\ldots,z_n$.

\begin{definition}\label{def:conformal-transducer}
  The \emph{conformal transducer} corresponding to a conformity measure $A$ is defined as
  \begin{multline}\label{eq:conformal-Q}
    Q(z_1,\ldots,z_n,(x_{n+1},y),\tau)
    :=
    \frac{1}{n+1}
    \left|\left\{i=1,\ldots,n+1\mid\alpha^y_i<\alpha^y_{n+1}\right\}\right|\\
    +
    \frac{\tau}{n+1}
    \left|\left\{i=1,\ldots,n+1\mid\alpha^y_i=\alpha^y_{n+1}\right\}\right|,
  \end{multline}
  where $n\in\{1,2,\ldots\}$, $(z_1,\ldots,z_n)\in\mathbf{Z}^n$ is training data,
  $x_{n+1}\in\mathbf{X}$ is a test predictor,
  and for each $y\in\R$ the corresponding \emph{conformity scores} $\alpha_i^y$ are defined by
  \begin{equation}\label{eq:conformity-scores}
    \begin{aligned}
      \alpha_i^y
      &:=
      A(z_1,\ldots,z_{i-1},z_{i+1},\ldots,z_n,(x_{n+1},y),z_i),
        \qquad i=1,\ldots,n,\\
      \alpha^y_{n+1}
      &:=
      A(z_1,\ldots,z_n,(x_{n+1},y)).
    \end{aligned}
  \end{equation}
  A function is a \emph{conformal transducer} if it is the conformal transducer corresponding to some conformity measure.
\end{definition}

The usual interpretation of \eqref{eq:conformal-Q} is as a randomized p-value
obtained when testing the IID model for the training data
extended by adding the test predictor $x_{n+1}$ combined with a postulated response $y$
(cf.\ Remark~\ref{rem:R2-intuition} at the end of this section).

\begin{definition}
  A \emph{conformal predictive system}
  is a function that is both a conformal transducer and a randomized predictive system.
  If $Q$ is a conformal predictive system,
  $Q(z_1,\ldots,z_n,(x_{n+1},\cdot),\tau)$ are the corresponding \emph{conformal predictive distributions}
  (or, more fully, conformal predictive distribution functions).
\end{definition}

\begin{example}\label{ex:Dempster-Hill}
  The simplest non-trivial conformal predictive system is a version of the classical Dempster--Hill procedure
  (to use the terminology of \cite{\OCMXVII};
  Dempster \cite{Dempster:1963} referred to it as direct probabilities
  and Hill \cite{Hill:1968,Hill:1988} as Bayesian nonparametric predictive inference,
  which was abbreviated to nonparametric predictive inference by Coolen \cite{Augustin/Coolen:2004}).
  The conformity measure is
  \begin{equation}\label{eq:Dempster-Hill}
    A(z_1,\ldots,z_n,(x_{n+1},y_{n+1}))
   :=
   y_{n+1},
  \end{equation}
  so that it ignores the predictors.
  Since the predictors are ignored, we will write $y_i$ in place of $z_i=(x_i,y_i)\in\mathbf{Z}$,
  omitting the predictors from our notation.
  Now suppose we are given training data $y_1,\ldots,y_n$
  and are interested in the conformal predictive distribution
  for the next response $y_{n+1}$;
  for simplicity, we will assume that $y_1,\ldots,y_n$ are all different.
  The conformity scores \eqref{eq:conformity-scores} are $\alpha^y_i=y_i$ and $\alpha^y_{n+1}=y$,
  and so the conformal predictive distribution is
  \[
    Q(y_1,\ldots,y_n,y,\tau)
    =
    \begin{cases}
      (i+\tau)/(n+1) & \text{if $y\in(y_{(i)},y_{(i+1)})$, $i=0,\ldots,n$}\\
      (i-1+2\tau)/(n+1) & \text{if $y=y_{(i)}$, $i=1,\ldots,n$},
    \end{cases}
  \]
  where $y_{(1)},\ldots,y_{(n)}$ is the sequence $y_1,\ldots,y_n$ sorted in the increasing order,
  $y_{(0)}:=-\infty$, and $y_{(n+1)}:=\infty$.
  A more intuitive (and equally informative) representation can be given in terms of the intervals
  \[
    Q(y_1,\ldots,y_n,y)
    :=
    \left[
      Q(y_1,\ldots,y_n,y,0),
      Q(y_1,\ldots,y_n,y,1)
    \right];
  \]
  namely,
  \begin{multline}\label{eq:Q-DH}
    Q(y_1,\ldots,y_n,y)
    =\\
    \begin{cases}
      [i/(n+1),(i+1)/(n+1)] & \text{if $y\in(y_{(i)},y_{(i+1)})$, $i=0,\ldots,n$}\\
      [(i-1)/(n+1),(i+1)/(n+1)] & \text{if $y=y_{(i)}$, $i=1,\ldots,n$}.
    \end{cases}
  \end{multline}
  For a further discussion of the Dempster--Hill procedure
  in the context of conformal prediction, see \cite{\OCMXVII}.
  For another example of a conformal predictive system
  (depending on the predictors in a simple but non-trivial way),
  see Example~\ref{ex:1NN}.
\end{example}

\begin{remark}\label{rem:R2}
  Requirement R2 in the previous section is sometimes referred to
  as the frequentist validity of predictive or confidence distributions
  (see, e.g., \cite{Xie/Singh:2013} and \cite{Shen/etal:2017}).
  It can be argued that there is no need to appeal to frequencies in these and similar cases
  (see, e.g., \cite{\GTPL}).
  However, the property of validity enjoyed by conformal predictive systems is truly frequentist:
  for them R2 (see \eqref{eq:R2}) can be strengthened to say that the random numbers
  $Q(z_1,\ldots,z_n,z_{n+1},\tau_n)$, $n=1,2,\ldots$,
  are distributed uniformly in $[0,1]$ and independently,
  provided $z_n\sim P$ and $\tau_n\sim U$, $n=1,2,\ldots$, are all independent
  \cite[Theorem~8.1]{Vovk/etal:2005book}.
  In combination with the law of large numbers this implies, e.g.,
  that for $\epsilon\in(0,1)$ the frequency of the event
  \[
    Q(z_1,\ldots,z_n,z_{n+1},\tau_n)
    \in
    \left[
      \frac{\epsilon}{2},
      1-\frac{\epsilon}{2}
    \right]
  \]
  (i.e., the frequency of the central $(1-\epsilon)$-prediction interval covering the true response)
  converges to $1-\epsilon$ as $n\to\infty$.
  Notice that this frequentist conclusion depends on the independence of $Q(z_1,\ldots,z_n,z_{n+1},\tau_n)$
  for different $n$;
  R2 alone is not sufficient.
\end{remark}

For a natural class of conformity measures the corresponding conformal transducers
are automatically conformal predictive systems.

\begin{definition}\label{def:monotonic}
  A conformity measure $A$ is \emph{monotonic} if $A(z_1,\ldots,z_{n+1})$ is:
  \begin{itemize}
  \item
    monotonically increasing in $y_{n+1}$,
    \begin{multline*}
      y_{n+1} \le y'_{n+1}
      \Longrightarrow
      A(z_1,\ldots,z_n,(x_{n+1},y_{n+1}))\\
      \le
      A(z_1,\ldots,z_n,(x_{n+1},y'_{n+1}));
    \end{multline*}
  \item
    monotonically decreasing in $y_{1}$,
    \begin{equation*}
      y_{1} \le y'_{1}
      \Longrightarrow
      A((x_1,y_1),z_{2},\ldots,z_n,z_{n+1})
      \ge
      A((x_1,y'_1),z_{2},\ldots,z_n,z_{n+1})
    \end{equation*}
    (which is equivalent to being decreasing in $y_i$ for any $i=2,\ldots,n$).
  \end{itemize}
\end{definition}

Let $A_n$ be the restriction of $A$ to $\mathbf{Z}^{n+1}$.

\begin{lemma}\label{lem:monotonic}
  Suppose a monotonic conformity measure $A$ satisfies the following three conditions:
  \begin{itemize}
  \item
    for all $n$, all training data sequences $(z_1,\ldots,z_n)$, and all test predictors $x_{n+1}$,
    \begin{align}
      \inf_{y}
      A(z_1,\ldots,z_n,(x_{n+1},y))
      &=
      \inf A_n,
      \label{eq:additional-inf}\\
      \sup_{y}
      A(z_1,\ldots,z_n,(x_{n+1},y))
      &=
      \sup A_n;
      \label{eq:additional-sup}
    \end{align}
  \item
    for each $n$,
    the $\inf_{y}$ in \eqref{eq:additional-inf} is either attained for all $(z_1,\ldots,z_n)$ and $x_{n+1}$
    or not attained for all $(z_1,\ldots,z_n)$ and $x_{n+1}$;
  \item
    for each $n$,
    the $\sup_{y}$ in \eqref{eq:additional-sup} is either attained for all $(z_1,\ldots,z_n)$ and $x_{n+1}$
    or not attained for all $(z_1,\ldots,z_n)$ and $x_{n+1}$.
  \end{itemize}
  Then the conformal transducer corresponding to $A$ is a randomized predictive system.
\end{lemma}

As usual, the two $\inf$ in \eqref{eq:additional-inf} are allowed to take value $-\infty$,
and the two $\sup$ in \eqref{eq:additional-sup} are allowed to take value $\infty$.
The conditions of Lemma~\ref{lem:monotonic} will be satisfied
if \eqref{eq:additional-inf} and \eqref{eq:additional-sup} hold
with $\inf A_n$ and $\sup A_n$ replaced by $-\infty$ and $\infty$, respectively;
we will usually use this simplified version of the lemma
(except for the proof of our main result,
where we will need a $[0,1]$-valued conformity measure).
Before proving Lemma~\ref{lem:monotonic},
we will give a less trivial example of a conformal predictive system
(cf.\ Example~\ref{ex:Dempster-Hill}).

\begin{example}\label{ex:1NN}
  In this example we will modify the conformity measure~\eqref{eq:Dempster-Hill}
  of the Dempster--Hill procedure by making it dependent, in a very simple way,
  on the predictors;
  it will satisfy all conditions of Lemma~\ref{lem:monotonic}
  (with $-\infty$ and $\infty$ on the right-hand sides of \eqref{eq:additional-inf} and \eqref{eq:additional-sup},
  respectively).
  Namely, we set
  \begin{equation}\label{eq:1NN}
    A
    \bigl(
      (x_1,y_1),\ldots,(x_n,y_n),(x_{n+1},y_{n+1})
    \bigr)
   :=
   y_{n+1} - \hat y_{n+1},
  \end{equation}
  where $\hat y_{n+1}$ is the response $y_i$ corresponding to the \emph{nearest neighbour} $x_i$ of $x_{n+1}$:
  $i\in\arg\min_{k\in\{1,\ldots,n\}}\rho(x_k,x_{n+1})$,
  $\rho$ being a measurable metric on the predictor space $\mathbf{X}$.
  In this example we only consider the case where the pairwise distances $\rho(x_i,x_j)$,
  $i,j\in\{1,\ldots,n+1\}$, are all different;
  the definition will be completed in Example~\ref{ex:1NN-r}.
  For each $i\in\{1,\ldots,n\}$, let $\hat y_i$ be the response $y_j$
  corresponding to the nearest neighbour $x_j$ to $x_i$ among $x_1,\ldots,x_n$:
  $j\in\arg\min_{k\in\{1,\ldots,n\}\setminus\{i\}}\rho(x_k,x_i)$.
  Let $I\subseteq\{1,\ldots,n\}$ be the set of $i\in\{1,\ldots,n\}$ such that $x_i$
  is closer to $x_{n+1}$ than to any of $x_j$, $j\in\{1,\ldots,n\}\setminus\{i\}$.
  The conformity scores \eqref{eq:conformity-scores} are $\alpha^y_{n+1}=y-\hat y_{n+1}$,
  where $\hat y_{n+1}$ is defined as in \eqref{eq:1NN},
  $\alpha^y_i=y_i-y$ if $i\in I$, and $\alpha^y_i=y_i-\hat y_i$ if $i\in \{1,\ldots,n\}\setminus I$.
  Solving the equation $\alpha^y_i=\alpha^y_{n+1}$ (cf.\ \eqref{eq:conformal-Q})
  gives $y=C_i:=(\hat y_{n+1}+y_i)/2$ if $i\in I$
  and
  \begin{equation}\label{eq:outside-I}
    y=C_i:=\hat y_{n+1}+(y_i-\hat y_i)
  \end{equation}
  if $i\in\{1,\ldots,n\}\setminus I$.
  Assuming, for simplicity, that $C_1,\ldots,C_n$ are all different,
  we obtain the conformal predictive distribution
  \begin{multline}\label{eq:Q-1NN}
    Q(y_1,\ldots,y_n,y)
    =\\
    \begin{cases}
      [i/(n+1),(i+1)/(n+1)] & \text{if $y\in(C_{(i)},C_{(i+1)})$, $i=0,\ldots,n$}\\
      [(i-1)/(n+1),(i+1)/(n+1)] & \text{if $y=C_{(i)}$, $i=1,\ldots,n$}
    \end{cases}
  \end{multline}
  (cf.\ \eqref{eq:Q-DH}),
  where $C_{(1)},\ldots,C_{(n)}$ is the sequence $C_1,\ldots,C_n$ sorted in the increasing order,
  $C_{(0)}:=-\infty$, and $C_{(n+1)}:=\infty$.

  The naive nearest-neighbour modification of the Dempster--Hill predictive distribution \eqref{eq:Q-DH}
  would be \eqref{eq:Q-1NN} with all $C_i$ defined by \eqref{eq:outside-I}.
  The conformal predictive distribution is different only for $i\in I$,
  and $I$ is typically a small set (its expected size is 1).
  For such $i$ the conformal predictive distribution modifies the residual
  $y_i-\hat y_i$ in \eqref{eq:outside-I} by replacing it by $(y_i-\hat y_{n+1})/2$.
  Intuitively, the nearest neighbour to $x_i$ in the augmented set $\{x_1,\ldots,x_{n+1}\}$ is $x_{n+1}$,
  so we would like to use $y_{n+1}$ instead of $\hat y_i$;
  but since we do not know $y_{n+1}$ as yet, we have to settle for its estimate $\hat y_{n+1}$,
  and the resulting loss of accuracy is counterbalanced by halving the new residual.
  This seemingly minor further modification
  ensures the small-sample property of validity R2.
\end{example}

\begin{example}\label{ex:LSPM}
  Another natural conformity measure is \eqref{eq:1NN}
  with $\hat y_{n+1}$ being the Least Squares prediction of $y_{n+1}$
  computed for the predictor $x_{n+1}$ given $z_1,\ldots,z_n$
  as training data;
  this makes $y_{n+1}-\hat y_{n+1}$ the deleted residual for $y_{n+1}$.
  Alternative definitions use ordinary residuals
  (where $(x_{n+1},y_{n+1})$ is added to the training data)
  and studentized residuals
  (which are half-way between deleted and ordinary residuals,
  in a certain sense).
  These conformity measures give rise to what is called Least Squares Prediction Machines
  in \cite{\OCMXVII}.
  Only the studentized version is a randomized predictive system;
  the other two versions satisfy property R1(i) only under the assumption
  of the absence of high-leverage points.
  See \cite{\OCMXVII} for an in-depth study of properties
  of Least Squares Prediction Machines,
  especially of their asymptotic efficiency under a standard Gaussian linear model.
\end{example}

\begin{remark}\label{rem:difference}
  The degree to which a randomized predictive system is affected by randomness,
  for given training data $(z_1,\ldots,z_n)$, test predictor $x_{n+1}$, and postulated response $y$,
  is \eqref{eq:difference}.
  As already mentioned, in interesting cases this difference will be small.
  For example, for the Dempster--Hill predictive system (Example~\ref{ex:Dempster-Hill}),
  the nearest neighbour predictive system (Example~\ref{ex:1NN}),
  and Least Squares Prediction Machines (Example~\ref{ex:LSPM}),
  the difference \eqref{eq:difference} is $1/(n+1)$ except for at most $n$ values of $y$,
  apart from pathological cases
  (see, e.g., \eqref{eq:Q-DH} and \eqref{eq:Q-1NN}).
  A randomized predictive system can be universally consistent
  only if the difference \eqref{eq:difference} is small with high probability.
\end{remark}

\begin{proof}[Proof of Lemma~\ref{lem:monotonic}]
  We need to check requirements R1 and R2.
  R2 is the standard property of validity for conformal transducers
  (see, e.g., \cite{Vovk/etal:2005book}, Theorem~8.1).
  The intuition behind the proof of this property is given in Remark~\ref{rem:R2-intuition}
  at the end of this section.

  The second statement of R1(i) is that \eqref{eq:conformal-Q} is monotonically increasing in $\tau$;
  this follows from \eqref{eq:conformal-Q} being a linear function of $\tau$
  with a nonnegative slope
  (the slope is in fact always positive as $i=n+1$ is allowed).

  The first statement of R1(i) is that \eqref{eq:conformal-Q} is monotonically increasing in $y$.
  We can rewrite \eqref{eq:conformal-Q} as
  \begin{equation}\label{eq:as-sum}
    Q(z_1,\ldots,z_n,(x_{n+1},y),\tau)
    =
    \frac{1}{n+1}
    \sum_{i=1}^{n+1}
    \left(
      1_{\{\alpha^y_i<\alpha^y_{n+1}\}}
      +
      \tau
      1_{\{\alpha^y_i=\alpha^y_{n+1}\}}
    \right),
  \end{equation}
  where $1_{\{E\}}$ stands for the indicator function of a property $E$,
  and it suffices to prove that each addend in~\eqref{eq:as-sum} is monotonically increasing in $y$;
  we will assume $i\le n$ (the case $i=n+1$ is trivial).
  This follows from $\alpha_i^y$ being monotonically decreasing in $y$
  and $\alpha_{n+1}^y$ being monotonically increasing in $y$,
  and therefore,
  \[
    1_{\{\alpha^y_i<\alpha^y_{n+1}\}}
    +
    \tau
    1_{\{\alpha^y_i=\alpha^y_{n+1}\}}
  \]
  taking all or some of the values $0$, $\tau$, $1$ in this order as $y$ increases.

  For concreteness, we will prove only the first statement of R1(ii), \eqref{eq:R1-first}.
  Fix an $n$.
  First let us assume that the $\inf_{y}$ in \eqref{eq:additional-inf} is attained for all $(z_1,\ldots,z_n)$ and $x_{n+1}$.
  We will have $\alpha^y_{n+1}=\inf A_n$ for sufficiently small $y$,
  and plugging $\tau:=0$ into \eqref{eq:conformal-Q} will give 0, as required.
  It remains to consider the case where the $\inf_{y}$ in \eqref{eq:additional-inf} is not attained
  for any $(z_1,\ldots,z_n)$ and $x_{n+1}$.
  Since $\min_{i=1,\ldots,n}\alpha^0_i>\inf A$,
  for sufficiently small $y$ we will have
  \[
    \alpha^y_{n+1}
    <
    \min_{i=1,\ldots,n}\alpha^0_i
    \le
    \min_{i=1,\ldots,n}\alpha^y_i,
  \]
  and so plugging $\tau:=0$ into \eqref{eq:conformal-Q} will again give 0.
\end{proof}

\begin{remark}\label{rem:R2-intuition}
  The proof of Lemma~\ref{lem:monotonic} refers to \cite{Vovk/etal:2005book}
  for a complete proof of R2.
  However, the intuition behind the proof is easy to explain.
  Setting $\tau:=1$ and assuming that there are no ties among the conformity scores,
  the right-hand side of \eqref{eq:conformal-Q} evaluated at $y:=y_{n+1}$
  is the rank of the last observation $(x_{n+1},y_{n+1})$ in the \emph{augmented training data}
  $(z_1,\ldots,z_n,(x_{n+1},y_{n+1}))$.
  Under the IID model (and the weaker assumption of the exchangeability
  of all the $n+1$ observations),
  the rank is uniformly distributed in the set $\{1,\ldots,n+1\}$.
  Dividing by $n+1$ and making $\tau\sim U$ leads to \eqref{eq:conformal-Q}
  (evaluated at $y:=y_{n+1}$)
  being uniformly distributed in $[0,1]$
  (even if some conformity scores are tied).
  This makes \eqref{eq:conformal-Q} a \emph{bona fide} randomized p-value
  for testing the IID model.
\end{remark}

\section{Mondrian predictive distributions}
\label{sec:MPD}

First we simplify our task by allowing Mondrian predictive distributions,
which are more general than conformal predictive distributions but enjoy the same property of validity R2.

\begin{definition}\label{def:taxonomy}
  A \emph{taxonomy} $\kappa$ is an equivariant measurable function
  that assigns to each sequence $(z_1,\ldots,z_n,z_{n+1})\in\mathbf{Z}^{n+1}$,
  for each $n\in\{1,2,\ldots\}$,
  an equivalence relation $\sim$ on $\{1,\ldots,n+1\}$.
\end{definition}

The requirement that $\kappa$ be equivariant will be spelled out in Definition~\ref{def:equivariant}.
The idea behind a taxonomy is to determine the comparison class for computing
the p-value \eqref{eq:conformal-Q};
instead of using all available data we will only use the observations
that are equivalent to the test observation
(intuitively, similar to it in some respect, with the aim of making the p-value more relevant).

The notation $(i\sim j \mid z_1,\ldots,z_{n+1})$, where $i,j\in\{1,\ldots,n+1\}$, means that $i$ is equivalent to $j$
under the equivalence relation assigned by $\kappa$ to $(z_1,\ldots,z_{n+1})$
(where $\kappa$ is always clear from the context and not reflected in our notation).
The measurability of $\kappa$ means that, for all $n$, $i$, and $j$, the set
$\{(z_1,\ldots,z_{n+1})\mid(i\sim j\mid z_1,\ldots,z_{n+1})\}$
is measurable.

\begin{definition}\label{def:equivariant}
  A permutation $\pi$ of $\{1,\ldots,n+1\}$ \emph{respects} an equivalence relation $\sim$
  if $\pi(i)\sim i$ for all $i=1,\ldots,n+1$.
  The requirement that a Mondrian taxonomy $\kappa$ be \emph{equivariant} means that,
  for each $n$, each $(z_1,\ldots,z_{n+1})\in\mathbf{Z}^{n+1}$,
  and each permutation $\pi$ of $\{1,\ldots,n+1\}$ respecting the equivalence relation
  assigned by $\kappa$ to $(z_1,\ldots,z_{n+1})$,
  we have
  \begin{equation}\label{eq:equivariance}
    (i\sim j \mid z_1,\ldots,z_{n+1})
    \Longrightarrow
    (\pi(i)\sim\pi(j)\mid z_{\pi(1)},\ldots,z_{\pi(n+1)}).
  \end{equation}
\end{definition}

\begin{remark}\label{rem:taxonomy}
  The notion of taxonomy used in this paper is introduced in \cite{\OCMVII} under the name of Venn taxonomies
  and subsumes Mondrian taxonomies as defined in \cite[Section 4.5]{Vovk/etal:2005book},
  Venn taxonomies as defined in \cite[Section 6.3]{Vovk/etal:2005book},
  and $n$-taxonomies as defined in \cite[Section 2.2]{Bala/etal:2014}.
  A narrower notion of taxonomy requires that \eqref{eq:equivariance}
  hold for all permutations $\pi$ of $\{1,\ldots,n+1\}$;
  the taxonomy of Section~\ref{sec:HMPD} belongs to this narrower class.
\end{remark}

\begin{definition}\label{def:Mondrian}
  Define
  \begin{equation*}
    \kappa(j \mid z_1,\ldots,z_{n+1})
    :=
    \left\{
      i\in\{1,\ldots,n+1\}
      \mid
      (i\sim j \mid z_1,\ldots,z_{n+1})
    \right\}
  \end{equation*}
  to be the equivalence class of $j$.
  The \emph{Mondrian transducer} corresponding to a taxonomy $\kappa$ and a conformity measure $A$ is
  \begin{multline}\label{eq:Mondrian-Q}
    Q(z_1,\ldots,z_n,(x_{n+1},y),\tau)\\
    :=
    \frac
    {
      \left|\left\{i\in\kappa(n+1\mid z_1,\ldots,z_n,(x_{n+1},y)) \mid
      \alpha^y_i<\alpha^y_{n+1}\right\}\right|
    }
    {
      \left|\kappa(n+1\mid z_1,\ldots,z_n,(x_{n+1},y))\right|
    }\\
    +
    \tau
    \frac
    {
      \left|\left\{i\in\kappa(n+1\mid z_1,\ldots,z_n,(x_{n+1},y)) \mid
      \alpha^y_i=\alpha^y_{n+1}\right\}\right|
    }
    {
      \left|\kappa(n+1\mid z_1,\ldots,z_n,(x_{n+1},y))\right|
    },
  \end{multline}
  where $n\in\{1,2,\ldots\}$, $(z_1,\ldots,z_n)\in\mathbf{Z}^n$ is training data,
  $x_{n+1}\in\mathbf{X}$ is a test predictor,
  and for each $y\in\mathbf{Y}$ the corresponding conformity scores $\alpha_i^y$ and $\alpha^y_{n+1}$
  are still defined by \eqref{eq:conformity-scores}.
  A function is a \emph{Mondrian transducer} if it is the Mondrian transducer
  corresponding to some taxonomy and conformity measure.
  A \emph{Mondrian predictive system} is a function that is both a Mondrian transducer
  and a randomized predictive system, as defined in Section~\ref{sec:RPD}.
\end{definition}

Notice that the denominator in~\eqref{eq:Mondrian-Q} is always positive.
The Mondrian p-value \eqref{eq:Mondrian-Q} differs from the original p-value \eqref{eq:conformal-Q}
in that it uses only the equivalence class of the test observation
(with a postulated response) as comparison class.
See \cite[Fig.~4.3]{Vovk/etal:2005book} for the origin of the attribute ``Mondrian''.

\begin{lemma}\label{lem:Mondrian}
  If a taxonomy does not depend on the responses
  and a conformity measure is monotonic and satisfies the three conditions of Lemma~\ref{lem:monotonic},
  the corresponding Mondrian transducer will be a randomized (and, therefore, Mondrian) predictive system.
\end{lemma}

\begin{proof}
  As in Lemma~\ref{lem:monotonic},
  the conformity scores (defined by \eqref{eq:conformity-scores})
  $\alpha^y_i$ are monotonically increasing in $y$ when $i=n+1$
  and monotonically decreasing in $y$ when $i=1,\ldots,n$.
  Since the equivalence class of $n+1$ in \eqref{eq:Mondrian-Q} does not depend on $y$,
  the value of \eqref{eq:Mondrian-Q} is monotonically increasing in $y$:
  it suffices to replace \eqref{eq:as-sum} by
  \begin{multline*}
    Q(z_1,\ldots,z_n,(x_{n+1},y),\tau)
    =
    \frac{1}{\left|\kappa(n+1\mid z_1,\ldots,z_n,(x_{n+1},y))\right|}\\
    \sum_{i\in\kappa(n+1\mid z_1,\ldots,z_n,(x_{n+1},y))}
    \left(
      1_{\{\alpha^y_i<\alpha^y_{n+1}\}}
      +
      \tau
      1_{\{\alpha^y_i=\alpha^y_{n+1}\}}
    \right)
  \end{multline*}
  in the argument of Lemma~\ref{lem:monotonic}.
  In combination with the obvious monotonicity in $\tau$, this proves R1(i).
  R1(ii) is demonstrated as in Lemma~\ref{lem:monotonic}.
  The proof of R2 is standard and valid for any taxonomy
  (see, e.g., \cite[Section 8.7]{Vovk/etal:2005book});
  the intuition behind it is given in Remark~\ref{rem:R2-intuition-Mondrian} below.
\end{proof}

The properties listed in Lemma~\ref{lem:Mondrian}
will be satisfied by the conformity measure and taxonomy
defined in Section~\ref{sec:HMPD} to prove Theorem~\ref{thm:MPS}, a weaker form of the main result of this paper.

\begin{remark}\label{rem:R2-intuition-Mondrian}
  Remark~\ref{rem:R2-intuition} can be easily adapted to Mondrian predictive systems.
  For $\tau:=1$ and assuming no ties among the conformity scores,
  the right-hand side of \eqref{eq:Mondrian-Q} at $y:=y_{n+1}$
  is the rank of the last observation $(x_{n+1},y_{n+1})$ in its equivalence class
  divided by the size of the equivalence class.
  Let us introduce another notion of equivalence:
  sequences $(z_1,\ldots,z_{n+1})$ and $(z'_1,\ldots,z'_{n+1})$ in $\mathbf{Z}^{n+1}$
  are \emph{equivalent} if
  \[
    (z'_1,\ldots,z'_{n+1})
    =
    \left(
      z_{\pi(1)},\ldots,z_{\pi(n+1)}
    \right)
  \]
  for some permutation $\pi$ of $\{1,\ldots,n+1\}$ that respects the equivalence relation
  assigned by $\kappa$ to $(z_1,\ldots,z_{n+1})$;
  this is indeed an equivalence relation since $\kappa$ is equivariant.
  The stochastic mechanism generating the augmented training data (the IID model)
  can be represented as generating an equivalence class (which is always finite)
  and then generating the actual sequence of observations in $\mathbf{Z}^{n+1}$
  from the uniform probability distribution on the equivalence class.
  Already the second step ensures that the rank is distributed uniformly
  in the set of its possible values,
  which leads to \eqref{eq:Mondrian-Q} being uniformly distributed in $[0,1]$,
  provided $y:=y_{n+1}$ and $\tau\sim U$.
\end{remark}

\begin{remark}\label{rem:advantage}
  One advantage of conformal predictive systems over Mondrian predictive systems
  is that the former satisfy a stronger version of R2, as explained in Remark~\ref{rem:R2}.
\end{remark}

\section{Universally consistent Mondrian predictive systems
  and probability forecasting systems}
\label{sec:UC-MPD}

Our results (Theorems~\ref{thm:MPS}--\ref{thm:CPS}) will assume that the predictor space $\mathbf{X}$ is standard Borel
(see, e.g., \cite[Definition 12.5]{Kechris:1995});
the class of standard Borel spaces is very wide and contains, e.g., all Euclidean spaces $\R^d$.
In this section we start from an easy result (Theorem~\ref{thm:MPS})
and its adaptation to deterministic forecasting (Theorem~\ref{thm:PFS}).

\begin{theorem}\label{thm:MPS}
  If the predictor space $\mathbf{X}$ is standard Borel,
  there exists a universally consistent Mondrian predictive system.
\end{theorem}

In Section~\ref{sec:HMPD} we will construct a Mondrian predictive system
that will be shown in Section~\ref{sec:proof-Mondrian} to be universally consistent.

Belyaev's generalization of weak convergence can also be applied in the situation
where we do not insist on small-sample validity;
for completeness, the following corollary of the proof of Theorem~\ref{thm:MPS}
covers this case.

\begin{definition}
  A \emph{probability forecasting system} is a measurable function
  $Q:\cup_{n=1}^{\infty}\mathbf{Z}^{n+1}\to[0,1]$
  such that:
  \begin{itemize}
  \item
    for each $n$, each training data sequence $(z_1,\ldots,z_n)\in\mathbf{Z}^n$, and each test predictor $x_{n+1}\in\mathbf{X}$,
    $Q(z_1,\ldots,z_n,(x_{n+1},y))$ is monotonically increasing in $y$;
  \item
    for each $n$, each training data sequence $(z_1,\ldots,z_n)\in\mathbf{Z}^n$, and each test predictor $x_{n+1}\in\mathbf{X}$,
    \begin{align*}
      \lim_{y\to-\infty}
      Q(z_1,\ldots,z_n,(x_{n+1},y))
      &=
      0,\\
      \lim_{y\to\infty}
      Q(z_1,\ldots,z_n,(x_{n+1},y))
      &=
      1;
    \end{align*}
  \item
    for each $n$, each training data sequence $(z_1,\ldots,z_n)\in\mathbf{Z}^n$, and each test predictor $x_{n+1}\in\mathbf{X}$,
    the function $Q(z_1,\ldots,z_n,(x_{n+1},\cdot))$ is right-continuous
    (and therefore, a \emph{bona fide} distribution function).
  \end{itemize}
  A probability forecasting system $Q$ is \emph{universally consistent} if,
  for any probability measure $P$ on $\mathbf{Z}$
  and any bounded continuous function $f:\R\to\R$,
  \eqref{eq:goal-1} holds in probability,
  where $Q_n:y\mapsto Q(z_1,\ldots,z_n,(x_{n+1},y))$,
  assuming $z_n\sim P$ are independent.
\end{definition}

\begin{theorem}\label{thm:PFS}
  If the predictor space $\mathbf{X}$ is standard Borel,
  there exists a universally consistent probability forecasting system.
\end{theorem}

\section{Histogram Mondrian predictive systems}
\label{sec:HMPD}

Remember that the measurable space $\mathbf{X}$ is assumed to be standard Borel.
Since every standard Borel space
is isomorphic to $\R$ or a countable set with discrete $\sigma$-algebra
(combine Theorems~13.6 and 15.6 in \cite{Kechris:1995}),
$\mathbf{X}$ is isomorphic to a Borel subset of $\R$.
Therefore, we can set, without loss of generality, $\mathbf{X}:=\R$,
which we will do.

\begin{definition}
  Fix a monotonically decreasing sequence $h_n$ of powers of 2 such that $h_n\to0$ and $n h_n\to\infty$ as $n\to\infty$.
  Let $\PPP_n$ be the partition of $\mathbf{X}$ into the intervals $[k h_n, (k+1)h_n)$,
  where $k$ are integers.
  We will use the notation $\PPP_n(x)$ for the interval (cell) of $\PPP_n$ that includes $x\in\mathbf{X}$.
  Let $A$ be the conformity measure defined by $A(z_1,\ldots,z_n,z_{n+1}):=y_{n+1}$,
  where $y_{n+1}$ is the response variable in $z_{n+1}$.
  This conformity measure will be called the \emph{trivial conformity measure}.
  The taxonomy under which $(i\sim j\mid z_1,\ldots,z_{n+1})$
  is defined to mean $x_j\in\PPP_n(x_i)$ is called the \emph{histogram taxonomy}.
\end{definition}

\begin{lemma}
  The trivial conformity measure is monotonic 
  and satisfies all other conditions of Lemma~\ref{lem:monotonic}.
  Therefore, the Mondrian transducer corresponding to it and the histogram taxonomy
  is a randomized predictive system.
\end{lemma}

\begin{proof}
  The infimum on the left-hand side of \eqref{eq:additional-inf} is always $-\infty$ and never attained,
  and the supremum on the left-hand side of \eqref{eq:additional-sup} is always $\infty$ and never attained.
  By definition, the histogram taxonomy does not depend on the responses.
  It remains to apply Lemma~\ref{lem:Mondrian}.
\end{proof}

\begin{definition}\label{def:HMPS}
  The Mondrian predictive system corresponding to the trivial conformity measure
  and histogram taxonomy is called the \emph{histogram Mondrian predictive system}.
\end{definition}

The histogram Mondrian predictive system will be denoted $Q$ in the next section,
where we will see that it is universally consistent.

\section{Proof of Theorem~\protect\ref{thm:MPS}}
\label{sec:proof-Mondrian}

Let us fix a probability measure $P$ on $\mathbf{Z}$;
our goal is to prove the convergence \eqref{eq:goal-1} in probability.
We fix a version of the conditional expectation $\Expect_P(f\mid x)$, $x\in\mathbf{X}$,
and use it throughout the rest of this paper.
We can split \eqref{eq:goal-1} into two tasks:
\begin{align}
  \Expect_P(f\mid\PPP_n(x_{n+1}))
  -
  \Expect_P(f\mid x_{n+1})
  &\to
  0,
  \label{eq:by-Levy}\\
  \int f\dd Q_n
  -
  \Expect_P(f\mid\PPP_n(x_{n+1}))
  &\to
  0,
  \label{eq:by-LLN}
\end{align}
where $\Expect_P(f\mid\PPP_n(x_{n+1}))$ is the conditional expectation of $f(y)$
given $x\in\PPP_n(x_{n+1})$ under $(x,y)\sim P$.

The convergence~\eqref{eq:by-Levy}
follows by Paul L\'evy's martingale convergence theorem
\cite[Theorem VII.4.3]{Shiryaev:2004};
Paul L\'evy's theorem is applicable since, by our assumption,
the partitions $\PPP_n$ are nested (as $h_n$ are powers of 2).
This theorem implies
$
  \Expect_P(f\mid\PPP_n(x))
  -
  \Expect_P(f\mid x)
  \to
  0
$
almost surely and, therefore, in probability when $(x,y)\sim P$.
The last convergence is clearly equivalent to~\eqref{eq:by-Levy}.

It remains to prove~\eqref{eq:by-LLN}.
Let $\epsilon>0$; we will show that
\begin{equation}\label{eq:sub-goal}
  \left|
    \int f\dd Q_n
    -
    \Expect_P(f\mid\PPP_n(x_{n+1}))
  \right|
  \le
  \epsilon
\end{equation}
with high probability for large enough $n$.
By \cite[the proof of Theorem~6.2]{Devroye/etal:1996},
the number $N$ of observations $z_i=(x_i,y_i)$ among $z_1,\ldots,z_n$ such that $x_i\in\PPP_n(x_{n+1})$
tends to infinity in probability.
Therefore, it suffices to prove that \eqref{eq:sub-goal}
holds with high conditional probability given $N>K$ for large enough $K$.
Moreover, it suffices to prove that, for large enough $K$,
\eqref{eq:sub-goal} holds with high conditional probability given $x_1,\ldots,x_{n+1}$
such that at least $K$ of predictors $x_i$ among $x_1,\ldots,x_n$ belong to $\PPP_n(x_{n+1})$.
(The remaining randomness is in the responses.)
Let $I\subseteq\{1,\ldots,n\}$ be the indices of those predictors;
remember that our notation for $\left|I\right|$ is $N$.
By the law of large numbers, the probability (over the random responses) of
\begin{equation}\label{eq:LLN}
  \left|
    \frac{1}{N} \sum_{i\in I} f(y_i)
    -
    \Expect_P(f\mid\PPP_n(x_{n+1}))
  \right|
  \le
  \epsilon/2
\end{equation}
can be made arbitrarily high by increasing $K$.
It remains to notice that
\begin{equation}\label{eq:last-nail}
  \int f\dd Q_n
  =
  \frac{1}{N+1} \sum_{i\in I} f(y_i);
\end{equation}
this follows from $\bar Q_n$ (in the notation of Section~\ref{sec:RPD})
being concentrated at the points $y_i$, $i\in I$, and assigning weight $a_i/(N+1)$ to each such $y_i$,
where $a_i$ is its multiplicity in the multiset $\{y_i\mid i\in I\}$
(our use of the same notation for sets and multisets is always counterbalanced by using unambiguous descriptors).
Interestingly, $\int f\dd Q_n$ in \eqref{eq:last-nail} does not depend on the random number $\tau$.

\section{Proof of Theorem~\protect\ref{thm:PFS}}
\label{sec:proof-PFS}

Define a probability forecasting system $Q$ by the requirement that $Q_n(\cdot):=Q(z_1,\ldots,z_n,(x_{n+1},\cdot))$
be the distribution function of the empirical probability measure of the multiset $\{y_i\mid i\in I\}$,
in the notation of the previous section.
In other words, the probability measure corresponding to $Q_n$ is concentrated on the set $\{y_i\mid i\in I\}$
and assigns the weight $a_i/N$ to each element $y_i$ of this set, where $a_i$ is its multiplicity in the multiset $\{y_i\mid i\in I\}$.
(This is very similar to $\bar Q_n$ at the end of the previous section.)
If $I=\emptyset$, let $Q_n(\cdot)$ be the distribution function of the probability measure concentrated at $0$.
We still have~\eqref{eq:LLN} with high probability,
and we have \eqref{eq:last-nail} with $N$ in place of $N+1$.

\section{Universally consistent conformal predictive systems}
\label{sec:UC-CPD}

In this section we will introduce a clearly innocuous extension of conformal predictive systems allowing further randomization.
In particular, the extension will not affect the small-sample property of validity, R2
(or its stronger version given in Remark~\ref{rem:R2}).

First we extend the notion of a conformity measure.

\begin{definition}
  A \emph{randomized conformity measure} is a measurable function
  $A:\cup_{n=1}^{\infty}(\mathbf{Z}\times[0,1])^{n+1}\to\R$ that
  is invariant with respect to permutations of extended training observations:
  for any $n$, any sequence $(z_1,\ldots,z_{n+1})\in\mathbf{Z}^{n+1}$,
  any sequence $(\theta_1,\ldots,\theta_{n+1})\in[0,1]^{n+1}$,
  and any permutation $\pi$ of $\{1,\ldots,n\}$,
  \begin{multline*}
    A
    \bigl(
      (z_1,\theta_1),\ldots,(z_n,\theta_n),(z_{n+1},\theta_{n+1})
    \bigr)\\
    =
    A
    \left(
      (z_{\pi(1)},\theta_{\pi(1)}),\ldots,(z_{\pi(n)},\theta_{\pi(n)}),(z_{n+1},\theta_{n+1})
    \right).
  \end{multline*}
\end{definition}

This is essentially the definition of Section~\ref{sec:CPD},
except that each observation is extended by adding a number (later it will be generated randomly from $U$)
that can be used for tie-breaking.
We can still use the same definition, given by the right-hand side of \eqref{eq:conformal-Q}, of the conformal transducer
corresponding to a randomized conformity measure $A$,
except for replacing each observation in \eqref{eq:conformity-scores} by an extended observation:
\begin{align*}
  \alpha_i^y
  &:=
  A
  \bigl(
    (z_1,\theta_{1}),\ldots,(z_{i-1},\theta_{i-1}),(z_{i+1},\theta_{i+1}),\ldots,(z_n,\theta_{n}),(x_{n+1},y,\theta_{n+1}),\\
    &\qquad\qquad (z_i,\theta_i)
  \bigr),
  \qquad i=1,\ldots,n,\\
  \alpha^y_{n+1}
  &:=
  A
  \bigl(
    (z_1,\theta_1),\ldots,(z_n,\theta_n),(x_{n+1},y,\theta_{n+1})
  \bigr).
\end{align*}

Notice that our new definition of conformal transducers is a special case of the old definition,
in which the original observation space $\mathbf{Z}$ is replaced by the extended observation space $\mathbf{Z}\times[0,1]$.
An extended observation $(z,\theta)=(x,y,\theta)$ will be interpreted to consist of an extended predictor $(x,\theta)$
and a response $y$.
The main difference from the old framework is that now we are only interested in the probability measures
on $\mathbf{Z}\times[0,1]$ that are the product of a probability measure $P$ on $\mathbf{Z}$
and the uniform probability measure $U$ on $[0,1]$.

The definitions of randomized predictive systems and monotonic conformity measures
generalize by replacing predictors $x_j$ by extended predictors $(x_j,\theta_j)$.
We still have Lemma~\ref{lem:monotonic}.
Conformal predictive systems are defined literally as before.

\begin{theorem}\label{thm:CPS}
  Suppose the predictor space $\mathbf{X}$ is standard Borel.
  There exists a universally consistent conformal predictive system.
\end{theorem}

In Section~\ref{sec:HCPD} we will construct a conformal predictive system
that will be shown in Section~\ref{sec:proof-conformal} to be universally consistent.
The corresponding randomized conformity measure will be monotonic
and satisfy all the conditions of Lemma~\ref{lem:monotonic}
(with predictors replaced by extended predictors).

\begin{example}\label{ex:1NN-r}
  This example will show the notion of a randomized conformity measure in action
  completing the definition in Example~\ref{ex:1NN}.
  Now we drop the assumption that the pairwise distances among $x_1,\ldots,x_{n+1}$
  are all different.
  We can use the same conformity measure~\eqref{eq:1NN},
  except that now the index $j$ of the nearest neighbour $x_j$ of $x_i$, $i\in\{1,\ldots,n+1\}$,
  is chosen randomly from the uniform probability measure on the set
  $\arg\min_{k\in\{1,\ldots,n\}\setminus\{i\}}\rho(x_k,x_i)$.
\end{example}

\section{Histogram conformal predictive systems}
\label{sec:HCPD}

In this section we will use the same partitions $\PPP_n$ of $\mathbf{X}=\R$ as in Section~\ref{sec:HMPD}.

\begin{definition}
  The \emph{histogram conformity measure} is defined to be the randomized conformity measure $A$
  with $A((z_1,\theta_1),\ldots,(z_n,\theta_n),(z_{n+1},\theta_{n+1}))$ defined as $a/N$,
  where $N$ is the number of predictors among $x_1,\ldots,x_n$ that belong to $\PPP_n(x_{n+1})$
  and $a$ is essentially the rank of $y_{n+1}$ among the responses corresponding to those predictors;
  formally,
  \[
    a
    :=
    \left|\left\{
      i=1,\ldots,n\mid x_i\in\PPP_n(x_{n+1}), (y_i,\theta_i) \le (y_{n+1},\theta_{n+1})
    \right\}\right|,
  \]
  where $\le$ refers to the lexicographic order
  (so that $(y_i,\theta_i) \le (y_{n+1},\theta_{n+1})$ means that either $y_i < y_{n+1}$
  or both $y_i=y_{n+1}$ and $\theta_i\le\theta_{n+1}$).
  If $N=0$, set, e.g.,
  \[
    A
    \bigl(
      (z_1,\theta_1),\ldots,(z_n,\theta_n),(z_{n+1},\theta_{n+1})
    \bigr)
    :=
    \begin{cases}
      1 & \text{if $y_{n+1}\ge0$}\\
      0 & \text{otherwise}.
    \end{cases}
  \]
\end{definition}

Since the histogram conformity measure is monotonic and satisfies all other conditions of Lemma~\ref{lem:monotonic}
(where now both $\inf$ and $\sup$ are always attained as 0 and 1, respectively),
the corresponding conformal transducer is a conformal predictive system.
In the next section we will show that it is universally consistent.

\section{Proof of Theorem~\protect\ref{thm:CPS}}
\label{sec:proof-conformal}

The proof in this section is an elaboration of the proof of Theorem~\ref{thm:MPS}
in Section~\ref{sec:proof-Mondrian}.
The difference is that now we have a different definition of $Q_n$.
It suffices to show that \eqref{eq:sub-goal} holds with probability at least $1-\epsilon$ for large enough $n$,
where $\epsilon>0$ is a given (arbitrarily small) positive constant.
In view of \eqref{eq:LLN}, it suffices to prove that
\begin{equation}\label{eq:new-goal}
  \left|
    \int f\dd Q_n
    -
    \frac{1}{N} \sum_{i\in I} f(y_i)
  \right|
  \le
  \epsilon/2
\end{equation}
holds with probability at least $1-\epsilon/2$ for large enough $n$.
In this section we are using the notation introduced in Section~\ref{sec:proof-Mondrian},
such as $N$ and $I$.

On two occasions we will use the following version of Markov's inequality applicable
to any probability space $(\Omega,\FFF,\Prob)$.
\begin{lemma}\label{lem:Markov}
  Let $\GGG$ be a sub-$\sigma$-algebra of $\FFF$ and $E\in\FFF$ be an event.
  For any positive constants $\delta_1$ and $\delta_2$,
  if $\Prob(E)\ge 1-\delta_1\delta_2$,
  then $\Prob(E\mid\GGG)>1-\delta_1$ with probability at least $1-\delta_2$.
\end{lemma}

\begin{proof}
  Assuming $\Prob(E)\ge 1-\delta_1\delta_2$,
  \begin{multline*}
    \Prob
    \Bigl(
      \Prob(E \mid \GGG)
      \le
      1-\delta_1
    \Bigr)
    =
    \Prob
    \Bigl(
      \Prob(E^c \mid \GGG)
      \ge
      \delta_1
    \Bigr)\\
    \le
    \frac{
      \Expect
      \left(
        \Prob(E^c \mid \GGG)
      \right)
    }
    {\delta_1}
    =
    \frac{\Prob(E^c)}{\delta_1}
    \le
    \frac{\delta_1\delta_2}{\delta_1}
    =
    \delta_2,
  \end{multline*}
  where $E^c$ is the complement of $E$ and the first inequality in the chain is a special case of Markov's.
\end{proof}

Set $C:=\sup\left|f\right|\vee10$.
Remember that $\epsilon>0$ is a given positive constant.
Let $B$ be so large that $y\in[-B,B]$ with probability at least $1-0.001\epsilon^2/C$
when $(x,y)\sim P$.
This is the first corollary of Lemma~\ref{lem:Markov} that we will need:

\begin{lemma}\label{lem:B}
  For a large enough $n$,
  the probability (over the choice of $z_1,\ldots,z_n,x_{n+1}$) of the fraction of $y_i$, $i\in I$,
  satisfying $y_i\in[-B,B]$ to be more than $1-0.02\epsilon/C$
  is at least $1-0.11\epsilon$.
\end{lemma}

\begin{proof}
  By Lemma~\ref{lem:Markov} we have
  \begin{equation}\label{eq:corollary-1}
    \Prob
    \left(
      \Prob(y\in[-B,B] \mid x\in\PPP_n(x')) > 1-0.01\epsilon/C
    \right)
    \ge
    1-0.1\epsilon,
  \end{equation}
  where the inner $\Prob$ is over $(x,y)\sim P$ and the outer $\Prob$ is over $x'\sim P_{\mathbf{X}}$,
  $P_{\mathbf{X}}$ being the marginal distribution of $P$ on the predictor space $\mathbf{X}$.
  To obtain the statement of the lemma it suffices to combine \eqref{eq:corollary-1}
  with the law of large numbers.
\end{proof}

Since $f$ is uniformly continuous over $[-B,B]$,
there is a partition
\[
  -B=y^*_0<y^*_1<\cdots<y^*_m<y^*_{m+1}=B
\]
of the interval $[-B,B]$ such that
\begin{equation}\label{eq:oscillation}
  \max_{y\in[f(y^*_{j}),f(y^*_{j+1})]} f(y)
  -
  \min_{y\in[f(y^*_{j}),f(y^*_{j+1})]} f(y)
  \le
  0.01\epsilon
\end{equation}
for $j=0,1,\ldots,m$.
Without loss of generality we will assume that $y\in\{y^*_0,\ldots,y^*_{m+1}\}$ with probability zero when $(x,y)\sim P$.
We will also assume, without loss of generality, that $m>10$.

Along with the conformal predictive distribution $Q_n$
we will consider the distribution function $Q^*_n$ of the multiset $\{y_i\mid i\in I\}$
(as defined in Section~\ref{sec:proof-PFS}, where it was denoted $Q_n$);
it exists only when $N>0$.
The next lemma will show that $Q_n$ is typically close to $Q^*_n$.
Let $K$ be an arbitrarily large positive integer.

\begin{lemma}\label{lem:close}
  For sufficiently large $n$,
  $Q_n(y^*_j)$ and $Q^*_n(y^*_j)$ (both exist and) differ from each other by at most $1/K+0.11\epsilon/C(m+1)+1/n$
  for all $j=0,1,\ldots,m+1$
  with probability (over the choice of $z_1,\ldots,z_n,x_{n+1}$ and random numbers $\tau,\theta_1,\ldots,\theta_{n+1}$)
  at least $1-0.11\epsilon$.
\end{lemma}

\begin{proof}
  We can choose $n$ so large that $N\ge K$ with probability at least $1-0.01\epsilon^2/C(m+1)(m+2)$.
  By Lemma~\ref{lem:Markov},
  for such $n$ the conditional probability that $N\ge K$ given $x_1,\ldots,x_n$ is at least $1-0.1\epsilon/C(m+1)$
  with probability (over the choice of $x_1,\ldots,x_n$) at least $1-0.1\epsilon/(m+2)$.
  Moreover, we can choose $n$ so large that the fraction of $x_i$, $i=1,\ldots,n$,
  which have at least $K-1$ other $x_i$, $i=1,\ldots,n$, in the same cell of $\PPP_n$
  is at least $1-0.11\epsilon/C(m+1)$ with probability at least $1-0.11\epsilon/(m+2)$
  (indeed, we can choose $n$ satisfying the condition in the previous sentence
  and generate sufficiently many new observations).

  Let us fix $j\in\{0,1,\ldots,m+1\}$.
  We will show that, for sufficiently large $n$,
  $Q_n(y^*_j)$ and $Q^*_n(y^*_j)$ differ from each other by at most $1/K+0.11\epsilon/C(m+1)+1/n$
  with probability at least $1-0.11\epsilon/(m+2)$.
  We will only consider the case $N>0$;
  we will be able to do so since the probability that $N=0$ tends to 0 as $n\to\infty$.
  The conformity score of the extended test observation $(x_{n+1},y^*_j,\theta_{n+1})$
  with the postulated response $y^*_j$ is, almost surely, $a/N$,
  where $a$ is the number of observations among $(x_i,y_i)$, $i\in I$,
  satisfying $y_i\le y^*_j$.
  (We could have written $y_i<y^*_j$ since we assumed earlier that $y=y^*_j$ with probability zero.)
  If a cell of $\PPP_n$ contains at least $K$ elements of the multiset $\{x_1,\ldots,x_n\}$,
  the percentage of elements of this cell with conformity score less than $a/N$
  is, almost surely, between $a/N-1/K$ and $a/N+1/K$;
  this remains true if ``less than'' is replaced by ``at most''.
  (It is here that we are using the fact that our conformity measure is randomized
  and, therefore, conformity scores are tied with probability zero.)
  And at most a fraction of $0.11\epsilon/C(m+1)$ of elements of the multiset $\{x_1,\ldots,x_n\}$
  are not in such a cell,
  with probability at least $1-0.11\epsilon/(m+2)$.
  Therefore, the overall percentage of elements of the multiset $\{x_1,\ldots,x_n\}$
  with conformity score less than $a/N$ is between $a/N-1/K-0.11\epsilon/C(m+1)$
  and $a/N+1/K+0.11\epsilon/C(m+1)$,
  with probability at least $1-0.11\epsilon/(m+2)$;
  this remains true if ``less than'' is replaced by ``at most''.
  Comparing this with the definition \eqref{eq:conformal-Q},
  we can see that $Q_n(y^*_j)$ is between $a/N-1/K-0.11\epsilon/C(m+1)-1/n$
  and $a/N+1/K+0.11\epsilon/C(m+1)+1/n$,
  with probability at least $1-0.11\epsilon/(m+2)$.
  It remains to notice that $Q^*_n(y^*_j)=a/N$ almost surely.
\end{proof}

Now we are ready to complete the proof of the theorem.
For sufficiently large $n$,
we can transform the left-hand side of \eqref{eq:new-goal} as follows
(as explained later):
\begin{align}
  &
  \left|
    \int f\dd Q_n
    -
    \frac{1}{N} \sum_{i\in I} f(y_i)
  \right|
  =
  \left|
    \int f \dd Q_n
    -
    \int f \dd Q^*_n
  \right|
  \label{eq:start}\\
  &\le
  \left|
    \int_{(-B,B]} f \dd Q_n
    -
    \int_{(-B,B]} f \dd Q^*_n
  \right|
  \notag\\
  &\qquad{}+
  C
  \bigl(
    Q^*_n(-B) + 1-Q^*_n(B)
    +
    Q_n(-B) + 1-Q_n(B)
  \bigr)
  \label{eq:step-1}\\
  &\le
  \left|
    \sum_{i=0}^m
    f(y^*_i)
    \left(
      Q_n(y^*_{i+1}) - Q_n(y^*_{i})
    \right)
    -
    \sum_{i=0}^m
    f(y^*_i)
    \left(
      Q^*_n(y^*_{i+1}) - Q^*_n(y^*_{i})
    \right)
  \right|
  \notag\\
  &\qquad{}+
  0.02\epsilon
  +
  C
  \left(
    0.08\frac{\epsilon}{C} + \frac2K + \frac{0.22 \epsilon}{C(m+1)} + \frac2n
  \right)
  \label{eq:step-2}\\
  &\le
  \sum_{i=0}^m
  \left|
    f(y^*_i)
  \right|
  \left|
    Q_n(y^*_{i+1}) - Q^*_n(y^*_{i+1})
    - Q_n(y^*_{i}) + Q^*_n(y^*_{i})
  \right|
  +
  0.2\epsilon
  \label{eq:step-3}\\
  &\le
  \sum_{i=0}^m
  \left|
    f(y^*_i)
  \right|
  \left(
    \frac{2}{K}
    +
    \frac{0.22 \epsilon}{C(m+1)}
    +
    \frac2n
  \right)
  +
  0.2\epsilon
  \label{eq:step-4}\\
  &\le
  \frac{2C(m+1)}{K}
  +
  0.42 \epsilon
  +
  \frac{2C(m+1)}{n}
  \le
  0.5 \epsilon.
  \label{eq:step-5}
\end{align}
Inequality~\eqref{eq:step-1} holds always.
Inequality~\eqref{eq:step-2} holds with probability
(over the choice of $z_1,\ldots,z_n$, $x_{n+1}$, and random numbers $\tau$ and $\theta_1,\ldots,\theta_{n+1}$)
at least $1-0.11\epsilon-0.11\epsilon=1-0.22\epsilon$
by \eqref{eq:oscillation} and Lemmas~\ref{lem:B} and~\ref{lem:close}:
the addend $0.02\epsilon$ arises by \eqref{eq:oscillation} from replacing integrals by sums,
the addend $0.08\epsilon/C$ is four times the upper bound on $Q^*_n(-B)$, or $1-Q^*_n(-B)$,
given by Lemma~\ref{lem:B}
(the factor of four arises from bounding $Q^*_n(-B)$, $1-Q^*_n(-B)$, $Q_n(-B)$, and $1-Q_n(-B)$),
and the expression $2/K+0.22\epsilon/C(m+1)+2/n$ arises from applying Lemma~\ref{lem:close}
to reduce bounding  $Q_n(-B)$ and $1-Q_n(-B)$ to bounding $Q^*_n(-B)$ and $1-Q^*_n(-B)$,
respectively.
Inequality~\eqref{eq:step-3} holds for sufficiently large $K$ and $n$.
Inequality~\eqref{eq:step-4} holds with probability at least $1-0.11\epsilon$ by Lemma~\ref{lem:close},
but this probability has already been accounted for.
And finally, the second inequality in~\eqref{eq:step-5} holds for sufficiently large $K$ and $n$.
Therefore, the whole chain \eqref{eq:start}--\eqref{eq:step-5}
holds with probability at least $1-0.22\epsilon \ge 1-\epsilon/2$.
This proves \eqref{eq:new-goal}, which completes the overall proof.

To avoid any ambiguity,
this paragraph will summarize the roles of $\epsilon$, $B$, $m$, $K$, and $n$ in this proof.
First we fix a positive constant $\epsilon>0$ (which, however, can be arbitrarily small).
Next we choose $B$, sufficiently large for the given $\epsilon$,
and after that, a sufficiently fine partition of $[-B,B]$ of size $m$.
We then choose $K$, which should be sufficiently large for the given $\epsilon$ and partition.
Finally, we choose $n$, which should be sufficiently large for the given $\epsilon$, partition, and $K$.

\section{Conclusion}
\label{sec:conclusion}

This paper constructs a universally consistent Mondrian predictive system
and, which is somewhat more involved, a universally consistent conformal predictive system.
There are many interesting directions of further research.
These are the most obvious ones:
\begin{itemize}
\item
  Investigate the best rate at which conformal predictive distributions
  and the true conditional distributions can approach each other.
\item
  Replace universal consistency by strong universal consistency
  (i.e., convergence in probability by convergence almost surely),
  perhaps in the online prediction protocol
  (as in Remark~\ref{rem:R2}).
\item
  Construct more natural, and perhaps even practically useful,
  universally consistent randomized predictive systems.
\end{itemize}

  \appendix
  \section{Marginal calibration}

  The main notion of validity (R2 in Definition~\ref{def:RPS}) used in this paper is,
  in the terminology of \cite[Definition 3(b)]{Gneiting/Katzfuss:2014},
  being probabilistically calibrated.
  The following definition gives another version of the calibration property
  \cite[Definition 3(a)]{Gneiting/Katzfuss:2014}
  for conformal predictive systems
  (of course, that version is applicable in a much wider context).
  The number of training observations will be referred to as the \emph{sample size}.

  \begin{definition}
    A conformal predictive system is \emph{marginally calibrated}
    for a sample size $n$ and a probability measure $P$ on $\mathbf{Z}^{n+1}$
    if, for any $y\in\R$,
    \begin{equation}\label{eq:marginal}
      \Expect
      \left(
        Q(z_1,\ldots,z_n,(x_{n+1},y),\tau)
      \right)
      =
      \Prob
      (y_{n+1} \le y),
    \end{equation}
    where both $\Expect$ and $\Prob$ are over $(z_1,\ldots,z_n,(x_{n+1},y_{n+1}),\tau)\sim P\times U$.
  \end{definition}

  In this appendix we will see that conformal predictive systems are not always
  marginally calibrated under the IID model.
  But we will start from an easier statement.

  The probabilistic calibration property R2 for a given sample size $n$
  depends only on the observations $(z_1,\ldots,z_{n+1})$
  being generated from an exchangeable distribution on $\mathbf{Z}^{n+1}$
  (and $\tau\sim U$ independently):
  see Remark~\ref{rem:R2-intuition} or, e.g., \cite[Theorem~8.1]{Vovk/etal:2005book}.
  The following example shows that there are conformal predictive systems that are not marginally calibrated
  for some sample size $n$ and an exchangeable probability measure on $\mathbf{Z}^{n+1}$,
  even among conformal predictive systems corresponding to conformity measures
  satisfying the conditions of Lemma~\ref{lem:monotonic}.

  \begin{example}\label{ex:no-marginal-1}
    Set $n:=1$, suppose $\left|\mathbf{X}\right|>1$,
    and let the data be generated from the exchangeable probability measure $P$ on $\mathbf{Z}^2$
    that assigns equal weights $1/2$ to the sequences $((x^{-1},-1),(x^{1},1))$ and $((x^{1},1),(x^{-1},-1))$
    in $\mathbf{Z}^2$,
    where $x^{-1}$ and $x^1$ are two distinct elements of $\mathbf{X}$
    (fixed for the rest of this appendix).
    Let a conformity measure $A$ satisfy
    \begin{equation}\label{eq:ex-A}
      A((x_1,y_1),(x_2,y_2))
      =
      \begin{cases}
        y_2 & \text{if $x_2=x^{1}$}\\
        3y_2+2 & \text{if $x_2=x^{-1}$};
      \end{cases}
    \end{equation}
    it is clear that $A$ can be extended to the whole of $\mathbf{X}$ and to all sample sizes $n$ in such a way
    that it satisfies all conditions in Lemma~\ref{lem:monotonic}.
    For $y=0$, the right-hand side of \eqref{eq:marginal} is $1/2$,
    whereas the left-hand side is different, namely $3/4$:
    \begin{itemize}
    \item
      with probability $1/2$ the training and test data
      form the sequence $((x^{-1},-1),(x^{1},1))$,
      and so the conformal predictive distribution is
      \begin{equation}\label{eq:-1,1}
        Q((x^{-1},-1),(x^{1},y))
        =
        \begin{cases}
          [0,1/2] & \text{if $y<-1$}\\
          [0,1] & \text{if $y=-1$}\\
          [1/2,1] & \text{if $y>-1$};
        \end{cases}
      \end{equation}
      the position $y=-1$ of the jump is found from the condition $\alpha^y_1=\alpha^y_2$,
      i.e., $3\times(-1)+2=y$;
    \item
      with probability $1/2$ the training and test data
      form the sequence $((x^{1},1),(x^{-1},-1))$,
      and so the conformal predictive distribution is
      \begin{equation}\label{eq:1,-1}
        Q((x^{1},1),(x^{-1},y))
        =
        \begin{cases}
          [0,1/2] & \text{if $y<-1/3$}\\
          [0,1] & \text{if $y=-1/3$}\\
          [1/2,1] & \text{if $y>-1/3$};
        \end{cases}
      \end{equation}
      the position $y=-1/3$ of the jump is found from the same condition $\alpha^y_1=\alpha^y_2$,
      which becomes $1=3y+2$;
    \item
      therefore, the mean value of the conformal predictive distribution at $y=0$ is $3/4$.
    \end{itemize}
  \end{example}

  Example~\ref{ex:no-marginal-1} can be strengthened by replacing the assumption of exchangeability
  by the IID model.

  \begin{example}\label{ex:no-marginal-2}
    Now we assume that the two observations are generated independently
    from the probability measure on $\mathbf{Z}$
    assigning equal weights $1/2$ to the observations $(x^{-1},-1)$ and $(x^{1},1)$.
    We consider the same conformity measure as in Example~\ref{ex:no-marginal-1}:
    see \eqref{eq:ex-A}.
    For $y=0$, the right-hand side of \eqref{eq:marginal} is still $1/2$,
    and the left-hand side becomes $5/8$:
    \begin{itemize}
    \item
      with probability $1/4$ the training and test data
      form the sequence $((x^{-1},-1),(x^{1},1))$,
      and so the conformal predictive distribution is \eqref{eq:-1,1},
      averaging $3/4$ at $y=0$;
    \item
      with probability $1/4$ the training and test data
      form the sequence $((x^{1},1),(x^{-1},-1))$,
      and so the conformal predictive distribution is \eqref{eq:1,-1},
      averaging $3/4$ at $y=0$;
    \item
      with probability $1/4$ the training and test data
      form the sequence $((x^{-1},-1),(x^{-1},-1))$,
      and so the conformal predictive distribution is
      \[
        Q((x^{-1},-1),(x^{-1},y))
        =
        \begin{cases}
          [0,1/2] & \text{if $y<-1$}\\
          [0,1] & \text{if $y=-1$}\\
          [1/2,1] & \text{if $y>-1$},
        \end{cases}
      \]
      which is $3/4$ on average at $y=0$;
      the position $y=-1$ of the jump is found from the condition $\alpha^y_1=\alpha^y_2$,
      which now is $3\times(-1)+2=3y+2$;
    \item
      finally, with probability $1/4$ the training and test data
      form the sequence $((x^{1},1),(x^{1},1))$,
      and so the conformal predictive distribution is
      \[
        Q((x^{1},1),(x^{1},y))
        =
        \begin{cases}
          [0,1/2] & \text{if $y<1$}\\
          [0,1] & \text{if $y=1$}\\
          [1/2,1] & \text{if $y>1$},
        \end{cases}
      \]
      which is $1/4$ on average at $y=0$;
      the position $y=1$ of the jump is found from the condition $\alpha^y_1=\alpha^y_2$,
      which now is $1=y$;
    \item
      therefore, the mean value of the conformal predictive distribution at $y=0$ is $5/8$.
    \end{itemize}
  \end{example}

  \section{Venn predictors}

  In \cite{\OCMXVII} and this paper,
  conformal prediction is adapted to probability forecasting.
  A traditional method of probability forecasting
  enjoying properties of validity similar to those of conformal prediction
  is Venn prediction \cite[Chapter~6]{Vovk/etal:2005book}.
  This appendix reviews Venn prediction and its properties of validity.
  We fix the sample size $n$.

  \begin{definition}\label{def:Venn}
    Let $\kappa$ be a taxonomy (as defined in Definition~\ref{def:taxonomy}).
    The \emph{Venn predictor} corresponding to $\kappa$
    is the family $\{Q_u\mid u\in\R\}$ of distribution functions
    defined by
    \begin{multline}\label{eq:Venn-Q}
      Q_u(z_1,\ldots,z_n,(x_{n+1},y))\\
      :=
      \frac
      {
        \left|\left\{i\in\kappa(n+1\mid z_1,\ldots,z_n,(x_{n+1},u)) \mid
        y_i\le y\right\}\right|
      }
      {
        \left|\kappa(n+1\mid z_1,\ldots,z_n,(x_{n+1},u))\right|
      }
    \end{multline}
    for any training data $(z_1,\ldots,z_n)$ and test predictor $x_{n+1}$.
  \end{definition}

  The definition~\eqref{eq:Venn-Q} is similar to, but simpler than, \eqref{eq:Mondrian-Q}.
  The intuition is that the Venn predictor contains
  (for $u:=y_{n+1}$ being the actual response in the test observation)
  the true empirical distribution function of the responses
  in the observations that are similar, in a suitable sense (determined by $\kappa$),
  to the test observation.
  The Venn prediction~\eqref{eq:Venn-Q} is useful when the distribution functions $Q_u$
  are close to each other for different $u$.
  Whereas this is a reasonable assumption for a suitable choice of $\kappa$
  in the case of classification
  (such as binary classification, $y_i\in\{0,1\}$, in \cite{\OCMVII}),
  in the case of regression it might make more sense to restrict attention to
  \[
    \left\{
      Q_u
      \mid
      u \in \Gamma^{\epsilon}(z_1,\ldots,z_n,x_{n+1})
    \right\}
  \]
  for a conformal predictor $\Gamma$ (see, e.g., \cite[Section~2.2]{Vovk/etal:2005book})
  and a small significance level $\epsilon>0$.

  The following theorem shows that Venn predictors are ideal
  in the technical sense of \cite[Definition~2.2]{Gneiting/Ranjan:2013}
  (and independent work by Tsyplakov).

  \begin{theorem}\label{thm:ideal}
    Let $\GGG$ be the $\sigma$-algebra on $\mathbf{Z}^{n+1}$
    consisting of the measurable subsets $E$ of $\mathbf{Z}^{n+1}$
    that are predictably invariant with respect to the taxonomy $\kappa$
    in the following sense:
    if a permutation $\pi$ of $\{1,\ldots,n+1\}$
    respects the equivalence relation $\sim$ assigned by $\kappa$ to $(z_1,\ldots,z_{n+1})$
    (in the sense of Definition~\ref{def:equivariant})
    and leaves $n+1$ in the same equivalence class,
    then
    \[
      (z_{1},\ldots,z_{n+1})\in E
      \Longrightarrow
      (z_{\pi(1)},\ldots,z_{\pi(n+1)})\in E.
    \]
    For any $y\in\R$,
    \begin{equation}\label{eq:ideal}
      Q_{y_{n+1}}(z_1,\ldots,z_n,(x_{n+1},y))
      =
      \Prob(y_{n+1}\le y\mid\GGG),
    \end{equation}
    where $(z_1,\ldots,z_n,(x_{n+1},y_{n+1}))$
    are generated from an exchangeable probability distribution on $\mathbf{Z}^{n+1}$.
  \end{theorem}

  Equation~\eqref{eq:ideal} expresses the condition of being \emph{ideal},
  with respect to some information base (namely, the $\sigma$-algebra $\GGG$).
  According to \cite[Theorem~2.8]{Gneiting/Ranjan:2013},
  this means that Venn predictors are both marginally and probabilistically calibrated,
  in the sense of one of their component distribution functions,
  namely $Q_{y_{n+1}}(z_1,\ldots,z_n,(x_{n+1},\cdot))$,
  being such.
  And according to \cite[Theorem~2.11]{Gneiting/Ranjan:2013},
  in the case of the binary response variable taking values in $\{0,1\}$,
  being probabilistically calibrated is equivalent to being \emph{conditionally calibrated}
  \begin{equation}\label{eq:conditionally}
    \Prob(y_{n+1}=1 \mid p_{n+1})
    =
    p_{n+1},
  \end{equation}
  where $p_{n+1}:=1-Q_{y_{n+1}}(z_1,\ldots,z_n,(x_{n+1},0))$
  is the predicted probability that $y_{n+1}=1$.
  Equation~\eqref{eq:conditionally} for Venn predictors, in the case of binary classification,
  is Theorem~1 in \cite{\OCMVII}.

  \begin{proof}[Proof of Theorem~\ref{thm:ideal}]
    Fix $y\in\R$.
    Let $P$ be the data-generating distribution
    (an exchangeable probability distribution on $\mathbf{Z}^{n+1}$),
    $Q$ be the random variable $Q_{y_{n+1}}(z_1,\ldots,z_n,(x_{n+1},y))$,
    and $E\in\GGG$.
    Notice that $Q$ is $\GGG$-measurable.
    Our goal is to prove
    \begin{equation}\label{eq:goal-2}
      \int_E 1_{\{y_{n+1}\le y\}} \dd P
      =
      \int_E Q \dd P,
    \end{equation}
    where $(z_1,\ldots,z_n,(x_{n+1},y_{n+1}))\sim P$.

    There are finitely many equivalence relations on the set $\{1,\ldots,n+1\}$.
    For each of them the set of data sequences $(z_1,\ldots,z_{n+1})$
    that are assigned this equivalence relation by the taxonomy $\kappa$ is measurable
    (by the requirement of measurability in the definition of a taxonomy)
    and, moreover, is an element of $\GGG$.
    Therefore, $E$ can be decomposed into a disjoint union of elements of $\GGG$
    all of whose elements are assigned the same equivalence relation by $\kappa$.
    We will assume, without loss of generality,
    that all elements of $E$ are assigned the same equivalence relation,
    which is fixed to the end of this proof.
    Let $\kappa(j)$ stand for the equivalence class of $j\in\{1,\ldots,n+1\}$.

    Let us say that two data sequences in $E$ are \emph{similar}
    if, for any equivalence class $C\subseteq\{1,\ldots,n+1\}$,
    they have the same numbers of observations with indices in $C$
    and with responses less than or equal to $y$.
    Following the same argument as in the previous paragraph,
    we further assume that all elements of $E$ are similar.

    Now we can see that both sides of \eqref{eq:goal-2} are equal to
    \[
      P(E)
      \frac
      {
        \left|\left\{i\in\kappa(n+1) \mid y_i\le y\right\}\right|
      }
      {
        \left|\kappa(n+1)\right|
      }
    \]
    (cf.\ \eqref{eq:Venn-Q}).
  \end{proof}
\end{document}